%% file: main.tex
\documentclass[letterpaper]{article} 
\usepackage{aaai25}  
\usepackage{times}  
\usepackage{helvet}  
\usepackage{courier}  
\usepackage[hyphens]{url}  
\usepackage{graphicx} 
\usepackage{natbib}  
\usepackage{caption} 
\usepackage{algorithm}
\usepackage{algorithmic}

\usepackage{newfloat}
\usepackage{listings}
\DeclareCaptionStyle{ruled}{labelfont=normalfont,labelsep=colon,strut=off} 
\floatstyle{ruled}
\newfloat{listing}{tb}{lst}{}
\floatname{listing}{Listing}
\usepackage{amsmath}
\usepackage{amssymb}
\usepackage{mathtools}
\usepackage{amsthm}
\usepackage{amsfonts}
\usepackage{booktabs}
\usepackage{dsfont}
\usepackage{macros}
\usepackage{microtype}
\usepackage{xspace}

\definecolor{Green}{RGB}{0, 152, 0}
\usepackage{hyperref}
\hypersetup{
    pdffitwindow=true,
    pdfstartview={FitH},
    pdfnewwindow=true,
    colorlinks,
    linktocpage=true,
    urlcolor=Green,
    linkcolor=Green,
    citecolor=Green
}
\usepackage[capitalize,noabbrev]{cleveref}

\title{Cross-Validated Off-Policy Evaluation}
\author {
    Matej Cief\textsuperscript{\rm 1,2},
    Branislav Kveton\textsuperscript{\rm 3},
    Michal Kompan\textsuperscript{\rm 2}
}
\affiliations {
    \textsuperscript{\rm 1}Brno University of Technology\\
    \textsuperscript{\rm 2}Kempelen Institute of Intelligent Technologies\\
    \textsuperscript{\rm 3}Adobe Research\textsuperscript{\rm *}
}

\begin{document}

\maketitle

\begin{abstract}
We study estimator selection and hyper-parameter tuning in off-policy evaluation. Although cross-validation is the most popular method for model selection in supervised learning, off-policy evaluation relies mostly on theory, which provides only limited guidance to practitioners. We show how to use cross-validation for off-policy evaluation. This challenges a popular belief that cross-validation in off-policy evaluation is not feasible. We evaluate our method empirically and show that it addresses a variety of use cases.
\end{abstract}

\frenchspacing

\input{01_introduction}
\input{02_off_policy_evaluation}
\input{03_related_work}
\input{04_cross_validation_in_ml}
\input{05_method}
\input{06_experiments}
\input{07_conclusion}

\section*{Acknowledgements}
This research was partially supported by DisAi, a project funded by the European Union under the Horizon Europe, GA No. 101079164, \url{https://doi.org/10.3030/101079164}; HERMES - a project by the EU NextGenerationEU through the Recovery and Resilience Plan for Slovakia under the project No. 09I03-03-V04-00336; and OZ BrAIn association.

\clearpage
\appendix

\input{08_appendix}

\bibliography{references}

\end{document}

%% file: 01_introduction.tex
\section{Introduction}

\renewcommand{\thefootnote}{\fnsymbol{footnote}}
\footnotetext[1]{The work was done at AWS AI Labs.}
\renewcommand{\thefootnote}{\arabic{footnote}}

\emph{Off-policy evaluation} \citep[OPE,][]{li_contextual-bandit_2010} is a framework for estimating the performance of a policy without deploying it online.
It is useful in domains where online A/B testing is costly or too dangerous. 
For example, deploying an untested algorithm in recommender systems or advertising can lead to a loss of revenue \citep{li_contextual-bandit_2010, silver_concurrent_2013}, and in medical treatments, it may have a detrimental effect on the patient's health \citep{hauskrecht_planning_2000}.
A popular approach to off-policy evaluation is \emph{inverse propensity scoring} \citep[IPS,][]{robins_estimation_1994}.
As this method is \emph{unbiased}, it approaches a true policy value with more data.

However, when the data logging policy has a low probability of choosing some actions, IPS-based estimates have a high \emph{variance} and often require a large amount of data to be useful in practice \citep{dudik_doubly_2014}.
Therefore, other lower-variance methods have emerged.
These methods often have hyper-parameters, such as a clipping constant to truncate large propensity weights \citep{ionides_truncated_2008}.
Some works provide theoretical insights \citep{ionides_truncated_2008, metelli_subgaussian_2021} for choosing hyper-parameters, while there are none for many others.

In supervised learning, data-driven techniques for hyper-parameter tuning, such as cross-validation, are more popular than theory-based techniques, such as the Akaike information criterion \citep{bishop_pattern_2006}.
The reason is that they perform better on large datasets, which are standard today.
Unlike in supervised learning, the ground truth value of the target policy is unknown in off-policy evaluation.
A common assumption is that standard machine learning approaches for model selection would fail because there is no unbiased and low-variance approach to compare estimators \citep{su_adaptive_2020}.
Therefore, only a few works studied estimator selection for off-policy evaluation, and no general solution exists \citep{saito_evaluating_2021, udagawa_policy-adaptive_2023}.

Despite common beliefs, we show that cross-validation in off-policy evaluation can be done comparably to supervised learning.
In supervised learning, we do not know the true data distribution, but we are given samples from it.
Each sample is an unbiased and high-variance representation of this distribution.
Nevertheless, we can still get an accurate estimate of true generalization when averaging the model error over these samples in cross-validation.
Similarly, we do not know the true reward distribution in off-policy evaluation, but we are given high-variance samples from it.
The difference is that these samples are biased because they are collected by a different policy.
However, we can use an unbiased estimator, such as IPS, on a held-out validation set to get an unbiased estimate of any policy value.
Then, as with supervised learning, we get an estimate of the estimator's performance.
Our contributions are:
\begin{itemize}
    \item We propose an easy-to-use estimator selection procedure for off-policy evaluation based on cross-validation that requires only data collected by a single policy.
    \item We analyze the loss of our procedure and how it relates to the true loss if the ground truth policy value was known. We use this insight to reduce its variance.
    \item We empirically evaluate the procedure on estimator selection and hyper-parameter tuning problems using nine real-world datasets. The procedure is more accurate than prior techniques and computationally efficient.
\end{itemize}

%% file: 02_off_policy_evaluation.tex
\section{Off-Policy Evaluation}
\label{sec: off-policy evaluation}

A contextual bandit \citep{langford_exploration_2008} is a popular model of an agent interacting with an unknown environment.
The interaction in round $i$ starts with the agent observing a \emph{context} $x_i \in \cX$, which is drawn i.i.d.\ from an unknown distribution $p$, where $\cX$ is the \emph{context set}.
Then the agent takes an \emph{action} $a_i \sim \pi(\cdot \mid x_i)$ from the \emph{action set} $\cA$ according to its policy $\pi$.
Finally, it receives a stochastic reward $r_i = r(x_i, a_i) + \varepsilon_i$, where $r(x, a)$ is the mean reward of action $a$ in context $x$ and $\varepsilon_i$ is an independent zero-mean noise.

In \emph{off-policy evaluation} \citep{li_contextual-bandit_2010}, a \emph{logging policy} $\pi_0$ interacts with the bandit for $n$ rounds and collects a \emph{logged dataset} $\cD = \set{(x_i, a_i, r_i)}_{i = 1}^n$.
The goal is to estimate the value of a \emph{target policy}
\begin{align*}
  V(\pi)
  = \sum_{x \in \cX} \sum_{a \in \cA} p(x) \pi(a \mid x) r(x, a)
\end{align*}
using the dataset $\cD$.
Various estimators have been proposed to either correct for the distribution shift caused by the differences in $\pi$ and $\pi_0$, or to estimate $r(x, a)$.
We review the canonical ones below and leave the rest to \cref{sec: appendix implementation details}.

The \emph{inverse propensity scores} estimator \citep[IPS,][]{robins_estimation_1994} reweights logged samples as if collected by the target policy $\pi$,
\begin{align}
  \label{eq: ips}
  \hat{V}_\textsc{IPS}(\pi; \cD)
  = \frac{1}{n} \sum_{i = 1}^n \frac{\pi(a_i \mid x_i)}{\pi_0(a_i \mid x_i)} r_i\,.
\end{align}
This estimator is unbiased but suffers from a high variance.
Therefore, a clipping constant is often used to truncate high propensity weights \citep{ionides_truncated_2008}.
This is a hyper-parameter that needs to be tuned.

The \emph{direct method} \citep[DM,][]{dudik_doubly_2014} is a popular approach to off-policy evaluation.
Using the DM, the policy value estimate can be computed as
\begin{align}
  \label{eq: dm}
  \hat{V}_\textsc{DM}(\pi; \cD)
  = \frac{1}{n} \sum_{i = 1}^n \sum_{a \in \cA} \pi(a \mid x_i) \hat{f}(x_i, a)\,,
\end{align} 
where $\hat{f}(x, a)$ is an estimate of the mean reward $r(x, a)$ from $\cD$. The function $\hat{f}$ is chosen from some function class, such as linear functions.

The \emph{doubly-robust} estimator \citep[DR,][]{dudik_doubly_2014} combines the DM and IPS as
\begin{align}
  \label{eq: dr}
  \hat{V}_\textsc{DR}(\pi; \cD)
  = {} & \frac{1}{n} \sum_{i = 1}^n \frac{\pi(a_i \mid x_i)}{\pi_0(a_i \mid x_i)}
  (r_i - \hat{f}(x_i, a_i)) + {} \\
  & \hat{V}_\textsc{DM}(\pi; \cD)\,,
  \nonumber
\end{align}
where $\hat{f}(x, a)$ is an estimate of $r(x, a)$ from $\cD$.
The DR is unbiased when the DM is, or the propensity weights are correctly specified.
The estimator is popular in practice because $r_i - \hat{f}(x_i, a_i)$ reduces the variance of rewards in the IPS part of the estimator.

Many other estimators with tunable parameters exist: TruncatedIPS \citep{ionides_truncated_2008}, \switchdr \citep{wang_optimal_2017}, Continuous OPE \citep{kallus_policy_2018}, CAB \citep{su_cab_2019}, DRos and DRps \citep{su_doubly_2020}, IPS-$\lambda$ \citep{metelli_subgaussian_2021}, MIPS \citep{saito_off-policy_2022}, Exponentially smooth IPS \citep{aouali_exponential_2023}, GroupIPS \citep{peng_offline_2023}, OffCEM \citep{saito_off-policy_2023}, Policy Convolution \citep{sachdeva_off-policy_2024}, Learned MIPS \citep{cief_learning_2024}, and subtracting control variates \citep{vlassis_design_2019}.
Some of these works leave the hyper-parameter selection as an open problem, while others provide a theory for selecting an optimal hyper-parameter, usually by bounding the bias of the estimator.
As in supervised learning, we show that theory is often too conservative, and given enough data, our method can select better hyper-parameters.
Other works use the statistical Lepski's adaptation method \citep{lepski_optimal_1997}, which requires that the hyper-parameters are ordered so that the bias is monotonically increasing.
The practitioner also needs to choose the estimator.
To address these shortcomings, we adapt cross-validation, a well-known machine learning technique for model selection, to estimator selection in a way that is general and applicable to \emph{any} estimator.

%% file: 03_related_work.tex
\section{Related Work}
\label{sec: related work}

To the best of our knowledge, there are only a few data-driven approaches for estimator selection or hyper-parameter tuning in off-policy evaluation for bandits. We review them below.

\citet{su_adaptive_2020} propose a hyper-parameter tuning method \slope based on Lepski's principle \citep{lepski_optimal_1997}.
The key idea is to order the hyper-parameter values so that the estimators' variances decrease.
Then, we compute the confidence intervals for all the values in this order.
If a confidence interval does not overlap with \emph{all} previous intervals, we stop and select the previous value.
While the method is straightforward, it assumes that the hyper-parameters are ordered such that the bias is monotonically increasing.
This makes it impractical for estimator selection, where it may be difficult to establish a correct order of the estimators.

\citet{saito_evaluating_2021} rely on a logged dataset collected by multiple logging policies.
They use one of the logging policies as the pseudo-target policy and directly estimate its value from the dataset.
Then, they choose the off-policy estimator that most accurately estimates the pseudo-target policy.
This approach assumes that we have access to a logged dataset collected by multiple policies.
Moreover, it ultimately chooses the best estimator for the pseudo-target policy, and not the target policy.
Prior empirical studies \citep{voloshin_empirical_2021} showed that the estimator's accuracy greatly varies when applied to different target policies.

In \pasif \citep{udagawa_policy-adaptive_2023}, two new surrogate policies are created using the logged dataset.
The surrogate policies have two properties: 1) the propensity weights from surrogate logging and target policies imitate those of the true logging and target policies, and 2) the logged dataset can be split in two as if each part was collected by one of the surrogate policies.
They learn a neural network that optimizes this objective.
Then, they evaluate estimators as in \citet{saito_evaluating_2021}, using surrogate policies and a precisely split dataset.
They show that estimator selection on these surrogate policies adapts better to the true target policy.

In this work, we do not require multiple logging policies, make no strong assumptions, and use principal techniques from supervised learning that are well-known and loved by practitioners.
Therefore, our method is easy to implement and, as showed in \cref{sec: experiments}, also more accurate.

A popular approach in offline policy selection \citep{lee_model_2022, nie_data-efficient_2022, saito_hyperparameter_2024} is to evaluate candidate policies on a held-out set by OPE.
\citet{nie_data-efficient_2022} even studied a similar approach to cross-validation.
While these papers seem similar to our work, the problems are completely different.
All estimators in our work estimate the same value $V(\pi)$, and this structure is used in the design of our solution (\cref{sec: method}).
We also address important questions that the prior works did not, such as how to choose a validator and how to choose the training-validation split.
A naive application of cross-validation without addressing these issues fails in OPE (Appendix B). 

%% file: 04_cross_validation_in_ml.tex
\section{Cross-Validation in Machine Learning}
\label{sec: cross-validation in machine learning}

Model selection \citep{bishop_pattern_2006} is a classic machine learning problem.
It can be addressed by two kinds of methods.
The first approach is probabilistic model selection, such as the Akaike information criterion \citep{parzen_information_1998} and Bayesian information criterion \citep{schwarz_estimating_1978}.
These methods penalize the complexity of the learned model during training \citep{stoica_model-order_2004}.
They are designed using theory and do not require a validation set.
Broadly speaking, they work well on smaller datasets because they favor simple models \citep{bishop_pattern_2006}.
The second approach estimates the performance of models on a held-out validation set, such as \emph{cross-validation} \citep[CV,][]{stone_cross-validatory_1974}.
CV is a state-of-the-art approach for large datasets and neural networks \citep{yao_early_2007}.
We focus on this setting because large amounts of logged data are available in modern machine learning.

In the rest of this section, we introduce cross-validation.
Let $f: \realset^d \to \realset$ be a function that maps features $x \in \realset^d$ to $\realset$.
It belongs to a function class $\cF$.
For example, $f$ is a linear function, and $\cF$ is the class of linear functions.
A machine learning algorithm $\alg$ maps a dataset $\cD$ to a function in $\cF$.
We write this as $f = \alg(\cF, \cD)$.
One approach to choosing $f$ is to minimize the \emph{squared loss} on $\cD$,
\begin{align*}
  L(f, \cD)
  = \sum_{(x, y) \in \cD} (y - f(x))^2\,,
\end{align*}
which can be written as
\begin{align}
  \alg(\cF, \cD)
  = \argmin_{f \in \cF} L(f, \cD)\,.
  \label{eq: overfitting}
\end{align}
This leads to overfitting on $\cD$ \citep{devroye_probabilistic_1996}.
To prevent this, CV is commonly used to evaluate $f$ on unseen validation data to give a more honest estimate of its generalization ability.
In $K$-fold CV, the dataset is split into $K$ folds.
We denote the validation data in the $k$-th fold by $\tilde{\cD}_k$ and all other training data by $\hat{\cD}_k$.
Using this notation, the average loss on a held-out set can be formally written as $\frac{1}{K} \sum_{k = 1}^K L(\alg(\cF, \hat{\cD}_k), \tilde{\cD}_k)$.

Cross-validation can be used to select a model as follows.
Suppose that we have a set of function classes $\mathbf{F} = \set{\cF}$.
For instance, $\mathbf{F} = \set{\cF_1, \cF_2}$, where $\cF_1$ is the class of linear functions and $\cF_2$ is the class of quadratic functions.
Then, the best function class under CV is
\begin{align}
  \label{eq: empirical risk minimization cv ml}
  \cF_*
  = \argmin_{\cF \in \mathbf{F}}
  \frac{1}{K} \sum_{k = 1}^K L(\alg(\cF, \hat{\cD}_k), \tilde{\cD}_k)\,.
\end{align}
After the best function class is chosen, a model is trained on the entire dataset as $f_* = \alg(\cF_*, \cD)$.

%% file: 05_method.tex
\section{Off-Policy Cross-Validation}
\label{sec: method}

Now, we adapt cross-validation to off-policy evaluation.
In supervised learning, we do not know the true data distribution but are given samples from it.
Each individual sample is an unbiased but noisy estimate of the true value.
Similarly, we do not know the true value of policy $\pi$ in off-policy evaluation.
However, we have samples collected by another policy $\pi_0$ and thus can estimate $V(\pi)$.

To formalize this observation, let $\tilde{V}(\pi; \tilde{\cD}_k)$ be an unbiased \emph{validator}, such as $\hat{V}_\textsc{IPS}$ or $\hat{V}_\textsc{DR}$ in \cref{sec: off-policy evaluation}, that estimates the true value from a \emph{validation set} $\tilde{\cD}_k$.
Let $\hat{V}(\pi; \hat{\cD}_k)$ be an \emph{evaluated estimator} on a \emph{training set} $\hat{\cD}_k$.
Then the squared loss of the evaluated estimator $\hat{V}_k = \hat{V}(\pi; \hat{\cD}_k)$ with respect to the validator $\tilde{V}_k = \tilde{V}(\pi; \tilde{\cD}_k)$ is
\begin{align}
  L(\hat{V}_k, \tilde{V}_k)
  = (\tilde{V}_k - \hat{V}_k)^2\,.
  \label{eq: loss ope}
\end{align}
Unlike in supervised learning (\cref{sec: cross-validation in machine learning}), the loss is only over one observation, an unbiased estimate of $V(\pi)$.
As in supervised learning, we randomly split the dataset $\cD$ into $\hat{\cD}_k$ and $\tilde{\cD}_k$, for $K$ times.
The average loss of an estimator $\hat{V}$ on a held-out validation set is $\frac{1}{K} \sum_{k = 1}^K L(\hat{V}_k, \tilde{V}_k)$.
In contrast to \cref{sec: cross-validation in machine learning}, we use \emph{Monte Carlo cross-validation} \citep{xu_monte_2001} because we need to control the sizes of $\hat{\cD}_k$ and $\tilde{\cD}_k$ independently from the number of splits.

The average loss on a held-out set can be used to select an estimator as follows.
Suppose that we have a set of estimators $\mathbf{V}$.
For instance, if $\mathbf{V} = \{\hat{V}_\textsc{IPS}, \hat{V}_\textsc{DM}, \hat{V}_\textsc{DR}\}$, the set contains IPS, DM, and DR (\cref{sec: off-policy evaluation}).
Then, the best estimator under CV can be defined similarly to \eqref{eq: empirical risk minimization cv ml} as
\begin{align}
  \hat{V}_*
  = \argmin_{\hat{V} \in \mathbf{V}} \frac{1}{K} \sum_{k = 1}^K
  L(\hat{V}_k, \tilde{V}_k)\,.
  \label{eq: cv ope}
\end{align}
After the best estimator is chosen, we return the estimated policy value from the entire dataset $\cD$, $\hat{V}_*(\pi; \cD)$.
This is the key idea in our proposed method.

To make the algorithm practical, we need to control the variances of the evaluated estimator and validator.
The rest of \cref{sec: method} contains an analysis that provides insights into this problem.
We also make \eqref{eq: cv ope} more robust.

\subsection{Analysis}
\label{sec: analysis}

We would like to choose an estimator that minimizes the true squared loss
\begin{align}
  (V(\pi) - \hat{V}(\pi))^2\,,
  \label{eq:true loss}
\end{align}
where $\hat{V}(\pi) = \hat{V}(\pi; \cD)$ is its evaluated estimate on dataset $\cD$ and $V(\pi)$ is the true policy value.
This cannot be done because $V(\pi)$ is unknown.
On the other hand, if $V(\pi)$ was known, we would not have an off-policy estimation problem.
In this analysis, we show that the minimized loss in \eqref{eq: cv ope} is a good proxy for \eqref{eq:true loss}.

We make the following assumptions.
The only randomness in our analysis is in how $\cD$ is split into the training set $\hat{\cD}_k$ and validation set $\tilde{\cD}_k$.
The sizes of these sets are $\hat{n}$ and $\tilde{n}$, respectively, and $\hat{n} + \tilde{n} = n$.
Let $\hat{V}_k = \hat{V}(\pi; \hat{\cD}_k)$ be the value of policy $\pi$ estimated by the evaluated estimator on $\hat{\cD}_k$ and $\hat{\mu} = \mathbb{E}[\hat{V}_k]$ be its mean.
Let $\tilde{V}_k = \tilde{V}(\pi; \tilde{\cD}_k)$ be the value of policy $\pi$ estimated by the validator on $\tilde{\cD}_k$ and $\tilde{\mu} = \mathbb{E}[\tilde{V}_k]$ be its mean.
Using this notation, the true loss in \eqref{eq:true loss} can be bounded from above as follows.

\begin{theorem}
\label{thm:true loss} For any split $k \in [K]$,
\begin{align*}
  (\hat{V}(\pi) - V(\pi))^2
  \leq {} & 2 \mathbb{E}[(\hat{V}_k - \tilde{V}_k)^2] + {} \\
  & 4 \mathbb{E}[(\hat{V}_k - \hat{\mu})^2] +
  4 \mathbb{E}[(\tilde{V}_k - \tilde{\mu})^2] + {} \\
  & 4 (\hat{\mu} - \hat{V}(\pi))^2 +
  4 (\tilde{\mu} - V(\pi))^2\,.
\end{align*}
\end{theorem}
\begin{proof}
The proof uses independence assumptions and that
\begin{align}
  (a + b)^2
  \leq 2 (a^2 + b^2)
  \label{eq:square of sum}
\end{align}
holds for any $a, b \in \realset$. As a first step, we introduce random $\hat{V}_k$ and $\tilde{V}_k$, and then apply \eqref{eq:square of sum},
\begin{align*}
  & (\hat{V}(\pi) - V(\pi))^2 \\
  & = \mathbb{E}[(\hat{V}(\pi) - \hat{V}_k + \hat{V}_k - V(\pi) +
  \tilde{V}_k - \tilde{V}_k)^2] \\
  & \leq 2 \mathbb{E}[(\hat{V}_k - \tilde{V}_k)^2] +
  2 \mathbb{E}[(\hat{V}(\pi) - \hat{V}_k - V(\pi) + \tilde{V}_k)^2]\,.
\end{align*}
Using \eqref{eq:square of sum} again, we bound the last term from above by
\begin{align*}
  4 \mathbb{E}[(\hat{V}_k - \hat{V}(\pi))^2] +
  4 \mathbb{E}[(\tilde{V}_k - V(\pi))^2]\,.
\end{align*}
Since $\hat{\mu} = \mathbb{E}[\hat{V}_k]$ and $\hat{\mu} - \hat{V}(\pi)$ is fixed, we get
\begin{align*}
  \mathbb{E}[(\hat{V}_k - \hat{V}(\pi))^2]
  & = \mathbb{E}[(\hat{V}_k - \hat{\mu} +
  \hat{\mu} - \hat{V}(\pi))^2] \\
  & = \mathbb{E}[(\hat{V}_k - \hat{\mu})^2] +
  (\hat{\mu} - \hat{V}(\pi))^2\,.
\end{align*}
Similarly, since $\tilde{\mu} = \mathbb{E}[\tilde{V}_k]$ and $\tilde{\mu} - V(\pi)$ is fixed, we get
\begin{align*}
  \mathbb{E}[(\tilde{V}_k - V(\pi))^2]
  = \mathbb{E}[(\tilde{V}_k - \tilde{\mu})^2] +
  (\tilde{\mu} - V(\pi))^2\,.
\end{align*}
Finally, we chain all inequalities and get our claim.
\end{proof}

\noindent The bound in \cref{thm:true loss} can be viewed as follows.
The first term $\mathbb{E}[(\hat{V}_k - \tilde{V}_k)^2]$ is the expectation of our optimized loss (\cref{thm:empirical loss}).
The second term is the variance of the evaluated estimator on a training set of size $\hat{n}$, and thus is $\cO(\hat{\sigma}^2 / \hat{n})$ for some $\hat{\sigma}^2 > 0$.
The third term is the variance of the validator on a validation set of size $\tilde{n}$, and thus is $\cO(\tilde{\sigma}^2 / \tilde{n})$ for some $\tilde{\sigma}^2 > 0$.
The fourth term is zero for any unbiased off-policy estimator in \cref{sec: off-policy evaluation}.
We assume that $\hat{\mu} = \hat{V}(\pi)$ in our discussion.
Finally, the last term is the difference between the unbiased estimate of the value of policy $\pi$ on $\cD$ and $V(\pi)$.
This term is $\cO(\log(1 / \delta) / n)$ with probability at least $1 - \delta$ by standard concentration inequalities \citep{boucheron_concentration_2013}, since $\cD$ is a sample of size $n$.
Based on our discussion,
\begin{align*}
  (\hat{V}(\pi) - V(\pi))^2
  \leq {} & 2 \mathbb{E}[(\hat{V}_k - \tilde{V}_k)^2] + {} \\
  & \cO(\hat{\sigma}^2 / \hat{n} + \tilde{\sigma}^2 / \tilde{n}) +
  \cO(\log(1 / \delta) / n)
\end{align*}
holds with probability at least $1 - \delta$.

The above bound can be minimized as follows.
The last term measures how representative the dataset $\cD$ is.
This is out of our control.
To minimize $\cO(\hat{\sigma}^2 / \hat{n} + \tilde{\sigma}^2 / \tilde{n})$, we set $\hat{n}$ and $\tilde{n}$ proportionally to the variances of the evaluated estimator and validator,
\begin{align*}
  \hat{n}
  = \frac{\hat{\sigma}^2}{\hat{\sigma}^2 + \tilde{\sigma}^2} n\,, \quad
  \tilde{n}
  = \frac{\tilde{\sigma}^2}{\hat{\sigma}^2 + \tilde{\sigma}^2} n\,,
\end{align*}
respectively.
Finally, we relate $\mathbb{E}[(\hat{V}_k - \tilde{V}_k)^2]$ to our minimized loss in \eqref{eq: cv ope}.

\begin{theorem}
\label{thm:empirical loss} For any split $\ell \in [K]$,
\begin{align*}
  \E{\frac{1}{K} \sum_{k = 1}^K (\hat{V}_k - \tilde{V}_k)^2}
  = \mathbb{E}[(\hat{V}_\ell - \tilde{V}_\ell)^2]\,.
\end{align*}
The variance of the estimator is
\begin{align*}
  \var{\frac{1}{K} \sum_{k = 1}^K (\hat{V}_k - \tilde{V}_k)^2}
  = \cO(1 / K)\,.
\end{align*}
\end{theorem}
\begin{proof}
The first claim follows from the linearity of expectation and that $\hat{V}_k - \tilde{V}_k$ are drawn independently from the same distribution.
To prove the second claim, we rewrite the variance of the estimator as
\begin{align}
  \frac{\sum_{i, j = 1}^K
  \mathbb{E}[(\hat{V}_i - \tilde{V}_i)^2 (\hat{V}_j - \tilde{V}_j)^2]}{K^2} -
  \mathbb{E}[(\hat{V}_k - \tilde{V}_k)^2]^2
  \label{eq:variance decomposition}
\end{align}
using $\var{X} = \mathbb{E}[X^2] - \mathbb{E}[X]^2$. Because the random splits are independent,
\begin{align*}
  \mathbb{E}[(\hat{V}_i - \tilde{V}_i)^2 (\hat{V}_j - \tilde{V}_j)^2]
  = \mathbb{E}[(\hat{V}_k - \tilde{V}_k)^2]^2
\end{align*}
for any $i \neq j$.
This happens exactly $K (K - 1)$ times out of $K^2$.
As a result, \eqref{eq:variance decomposition} can be rewritten as
\begin{align*}
  \frac{1}{K} \mathbb{E}[(\hat{V}_k - \tilde{V}_k)^4] -
  \frac{1}{K} \mathbb{E}[(\hat{V}_k - \tilde{V}_k)^2]^2
  = \cO(1 / K)\,.
\end{align*}
This concludes the proof.
\end{proof}

\noindent \cref{thm:empirical loss} says that the estimated loss from $K$ random splits concentrates at $\mathbb{E}[(\hat{V}_k - \tilde{V}_k)^2]$ at rate $\cO(1 / \sqrt{K})$.
Hence, by standard concentration inequalities \citep{boucheron_concentration_2013},
\begin{align*}
  (\hat{V}(\pi) - V(\pi))^2
  \leq {} & \frac{2}{K} \sum_{k = 1}^K (\hat{V}_k - \tilde{V}_k)^2 + {} \\
  & \cO(\hat{\sigma}^2 / \hat{n} + \tilde{\sigma}^2 / \tilde{n}) + {} \\
  & \cO(\log(1 / \delta) / n) + \cO(\sqrt{\log(1 / \delta') / K})
\end{align*}
holds with probability at least $1 - \delta - \delta'$.
The last term can be driven to zero with more random splits $K$.

\subsection{One Standard Error Rule}
\label{sec: one standard error rule}

If the set of estimators $\mathbf{V}$ in \eqref{eq: cv ope} is large, we could choose a poor estimator that performs well just by chance with a high probability.
This problem is exacerbated in small datasets \citep{varma_bias_2006}.
To account for this in supervised CV, \citet{hastie_elements_2009} proposed a heuristic called the \emph{one standard error rule}.
This heuristic chooses the simplest model whose performance is within one standard error of the best model.
Roughly speaking, these models cannot be statistically distinguished.

Inspired by the one standard error rule, we choose an estimator with the \emph{lowest one-standard-error upper bound} on its loss.
This is also known as \emph{pessimistic optimization} \citep{buckman_importance_2020}.
Compared to the original rule \citep{hastie_elements_2009}, we do not need to know which estimator has the lowest complexity.

\begin{algorithm}[t]
  \caption{Off-policy evaluation with cross-validated estimator selection.}
  \label{alg: cross-validation}
  \begin{algorithmic}[1]
    \STATE \textbf{Input:} Evaluated policy $\pi$, logged dataset $\cD$, set of estimators $\mathbf{V}$, number of random splits $K$
    \STATE $\tilde{\sigma}^2 \gets
    \text{Empirical estimate of $\var{\tilde{V}(\pi;\cD)}$}$
    \FOR{$\hat{V} \in \mathbf{V}$}
      \STATE $\hat{\sigma}^2 \gets
      \text{Empirical estimate of $\var{\hat{V}(\pi; \cD)}$}$
      \FOR{$k = 1, \dots, K$}
        \STATE $\hat{\cD}_k, \tilde{\cD}_k \gets
        \text{Split $\cD$ such that $|\hat{\cD}_k| / |\tilde{\cD}_k| =
        \hat{\sigma}^2 / \tilde{\sigma}^2$}$
        \STATE $L_{\hat{V}, k} \gets
        (\tilde{V}(\pi; \tilde{\cD}_k) - \hat{V}(\pi; \hat{\cD}_k))^2$
      \ENDFOR
      \STATE $\bar{L}_{\hat{V}} \gets
      \frac{1}{K} \sum_{k = 1}^K L_{\hat{V}, k}$
    \ENDFOR
    \STATE $\displaystyle \hat{V}_* \gets
    \argmin_{\hat{V} \in \mathbf{V}} \bar{L}_{\hat{V}} +
    \sqrt{\frac{1}{K - 1} \sum_{k = 1}^K (L_{\hat{V}, k} - \bar{L}_{\hat{V}})^2}$
    \STATE \textbf{Output:} $\hat{V}_*(\pi; \cD)$
  \end{algorithmic}
\end{algorithm}

\subsection{Algorithm}
\label{sec:algorithm}

We call our method \textbf{O}ff-policy \textbf{C}ross-\textbf{V}alidation (\ocv) and present its pseudo-code in \cref{alg: cross-validation}.
The method works as follows.
First, we estimate the variance of the validator $\tilde{V}$ (line 2).
Details are provided in \cref{sec: appendix variance estimation}.
Second, we estimate the variance of each evaluated estimator $\hat{V}$ (line 4).
Third, we repeatedly split $\cD$ into the training and validation sets (line 6) and calculate the loss of the evaluated estimator with respect to the validator (line 7).
Finally, we select the estimator $\hat{V}_*$ with the lowest one-standard-error upper bound on its estimated loss (line 11).

%% file: 06_experiments.tex
\begin{figure*}
  \centering
  \includegraphics[width=\linewidth]{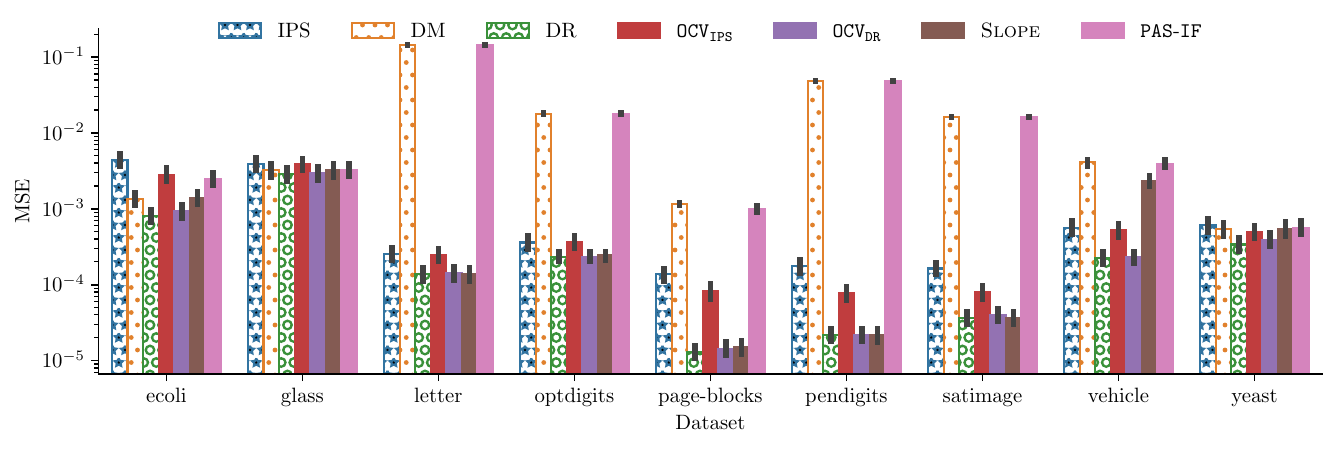}
  \caption{MSE of our estimator selection methods, \ocvips and \ocvdr, compared against two other estimator selection baselines, \slope and \pasif. The methods select the best estimator out of IPS, DM, and DR. In all figures, we report $95\%$ confidence intervals estimated by bootstrapping.}
  \label{fig: UCI estimator selection}
\end{figure*}

\begin{figure*}
  \centering
  \includegraphics[width=\linewidth]{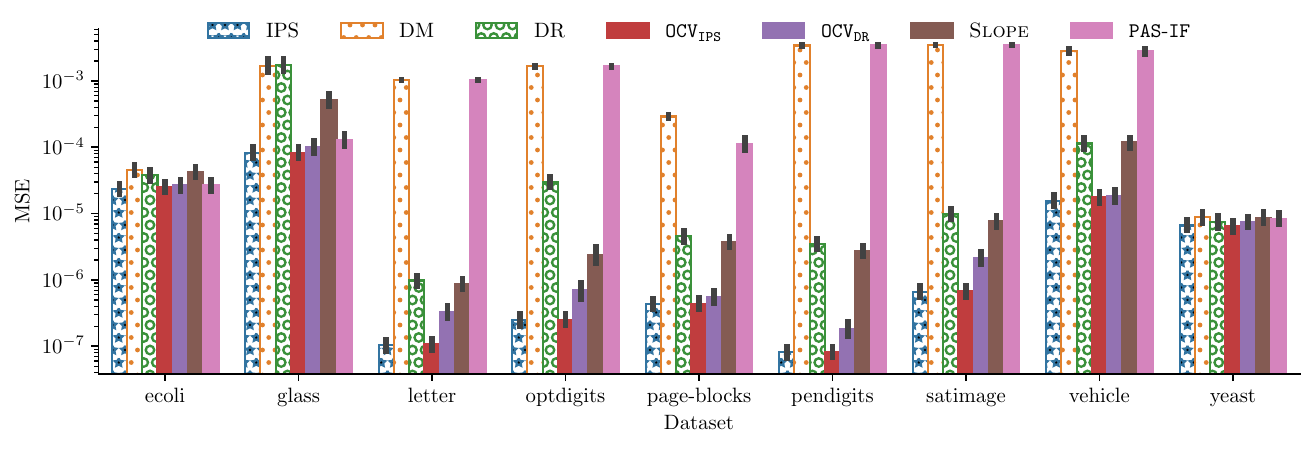}
  \caption{MSE of the methods for temperatures $\beta_0 = 1$ and $\beta_1 = -10$. \ocv performs well even when its validator does not, for example \ocvdr on the \emph{glass} dataset. This also shows that \ocv does not simply choose the same estimator as the validator.}
  \label{fig: UCI estimator selection beta 1 -10}
\end{figure*}

\begin{figure*}
  \centering
  \includegraphics[width=\linewidth]{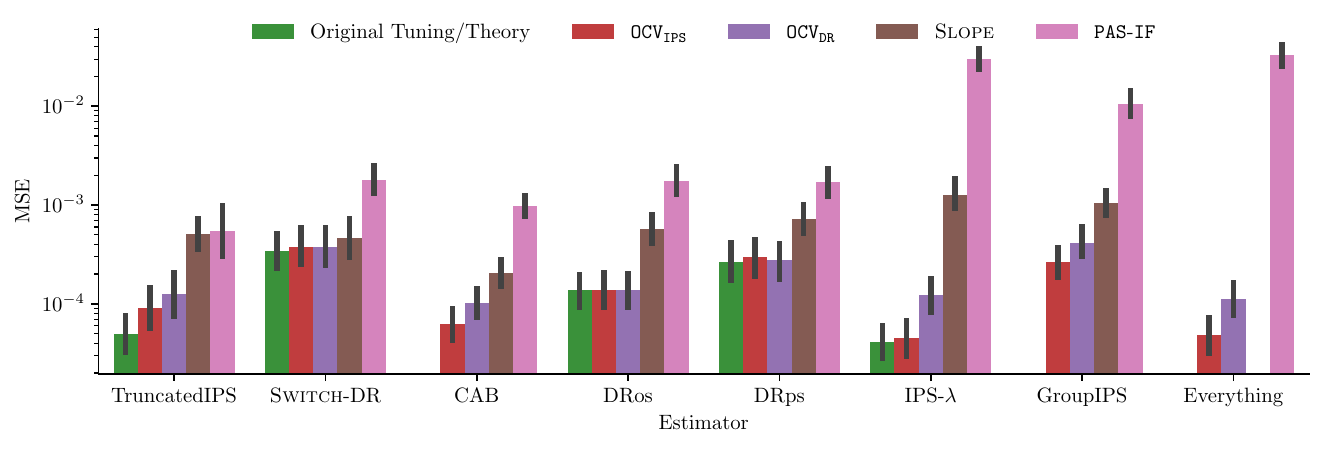}
  \caption{MSE of our estimator selection methods and specialized theoretical approaches applied to hyper-parameter tuning of various estimators. \emph{Everything} refers to the joint estimator selection and hyper-parameter tuning. This shows that \ocv is a reliable and practical method for choosing a suitable and well-tuned estimator.}
  \label{fig: tuning}
\end{figure*}

\section{Experiments}
\label{sec: experiments}

We conduct three main experiments.
First, we evaluate \ocv on an estimator selection problem among IPS, DM, and DR.
Second, we apply \ocv to hyper-parameter tuning of seven other estimators.
We compare against \slope, \pasif, and estimator-specific tuning heuristics if the authors provided one.
Finally, we show that \ocv can jointly choose the best estimator and its hyper-parameters, and thus is a practical method to get a high-quality estimator.
\cref{sec: appendix additional experiments} contains ablation studies on the individual components of \ocv and computational efficiency.
We also show the importance of having an unbiased validator and that \ocv performs well even in low-data regimes. \\

\noindent \textbf{Datasets.}
We take nine UCI ML Repository datasets \citep{bache_uci_2013} and convert them into contextual bandit problems, similarly to prior works \citep{dudik_doubly_2014,wang_optimal_2017,farajtabar_more_2018,su_cab_2019,su_doubly_2020}.
The datasets have different characteristics (Appendix A), 
such as sample size and the number of features, and thus cover a wide range of potential applications of our method.
Each dataset contains $n$ examples, $\cH = \set{(x_i, y_i)}_{i \in [n]}$, where $x_i \in \realset^d$ and $y_i \in [m]$ are the feature vector and label of example $i$, respectively; and $m$ is the number of classes.
We split each $\cH$ into two halves, the bandit feedback dataset $\cH_b$ and policy learning dataset $\cH_\pi$.

The \emph{bandit feedback dataset} is used to compute the policy value and log data.
Specifically, the value of policy $\pi$ is
\begin{align*}
  V(\pi)
  = \frac{1}{|\cH_b|} \sum_{(x, y) \in \cH_b} \sum_{a = 1}^m \pi(a \mid x) \I{a = y}\,.
\end{align*}
The logged dataset $\cD$ has the same size as $\cH_b$, $n = |\cH_b|$, and is defined as
\begin{align*}
  \cD
  = \set{(x, a, \I{a = y}): a \sim \pi_0(\cdot \mid x), \, (x, y) \in \cH_b}\,.
\end{align*}
For each example in $\cH_b$, the logging policy $\pi_0$ takes an action conditioned on its feature vector.
The reward is one if the index of the action matches the label and zero otherwise.

The \emph{policy learning dataset} is used to estimate $\pi$ and $\pi_0$.
We proceed as follows.
First, we take a bootstrap sample of $\cH_\pi$ of size $|\cH_\pi|$ and learn a logistic model for each class $a \in [m]$.
Let $\theta_{a, 0} \in \realset^d$ be the learned logistic model parameter for class $a$.
Second, we take another bootstrap sample of $\cH_\pi$ of size $|\cH_\pi|$ and learn a logistic model for each class $a \in [m]$.
Let $\theta_{a, 1} \in \realset^d$ be the learned logistic model parameter for class $a$ in the second bootstrap sample.
Based on $\theta_{a, 0}$ and $\theta_{a, 1}$, we define our policies as
\begin{align}
\begin{split}
  \label{eq: experiments policy softmax}
  \pi_0(a \mid x)
  = \frac{\exp(\beta_0 x^\top \theta_{a, 0} )}
  {\sum_{a' = 1}^m \exp(\beta_0 x^\top \theta_{a', 0} )}\,, \\
  \pi(a \mid x)
  = \frac{\exp(\beta_1 x^\top \theta_{a, 1} )}
  {\sum_{a' = 1}^m \exp(\beta_1 x^\top \theta_{a', 1} )}\,.
  \end{split}
\end{align}
The parameters $\beta_0$ and $\beta_1$ are \emph{inverse temperatures} of the softmax function.
Positive values prefer high-value actions and vice versa.
The zero temperature is a uniform policy.
The temperatures $\beta_0$ and $\beta_1$ are chosen later in each experiment.
We take two bootstrap samples to ensure that $\pi$ and $\pi_0$ are not simple transformations of each other. \\

\noindent \textbf{Our method and baselines.}
We evaluate two variants of our method, \ocvips and \ocvdr, with IPS and DR as validators.
\ocv is implemented as described in \cref{alg: cross-validation} with $K = 10$.
The reward model $\hat{f}$ in all relevant estimators is learned using ridge regression with a regularization coefficient $0.001$.
We consider two baselines: \slope and \pasif (\cref{sec: related work}).
In the tuning experiment (\cref{sec: experiment hyper-parameter tuning}), we also implement the original tuning procedures if the authors provided one.
All implementation details are in \cref{sec: appendix tuning methods details}.

\subsection{Estimator Selection}
\label{sec: experiment estimator selection}

We want to choose the best estimator from three candidates: IPS in \eqref{eq: ips}, DM in \eqref{eq: dm}, and DR in \eqref{eq: dr}.
We use $\beta_0 = 1$ for the logging policy and $\beta_1 = 10$ for the target policy.
This is a realistic scenario where the logging policy prefers high-value actions, and the target policy takes them even more often.
\slope requires the estimators to be ordered by their variances.
This may not be always possible.
However, the bias-variance trade-offs of IPS, DM, and DR are generally
\begin{align*}
  \var{\hat{V}_\textsc{IPS}(\pi)}
  \geq \var{\hat{V}_\textsc{DR}(\pi)}
  \geq \var{\hat{V}_\textsc{DM}(\pi)}
\end{align*}
and we use this order.
All our results are averaged over $500$ independent runs.
A new run always starts by splitting the dataset into the bandit feedback and policy learning datasets, as described earlier. \\

\noindent \textbf{Cross-validation consistently chooses a good estimator.}
\cref{fig: UCI estimator selection} shows that our methods avoid the worst estimator and perform better on average than both \slope and \pasif.
\ocvdr significantly outperforms all methods on two datasets while never being much worse.
We observe that \slope performs well because its bias-variance assumptions are satisfied.
\pasif prefers DM even though it performs poorly.
We hypothesize that this is because the tuning procedure of \pasif is biased.
As we show in \cref{sec: appendix additional experiments}, a biased validator tends to prefer similarly biased estimators and thus cannot be reliably used for estimator selection. \\

\noindent \textbf{Cross-validation with DR performs well even when DR performs poorly.}
One may think that \ocvdr performs well in \cref{fig: UCI estimator selection} because the best estimator is DR \citep{dudik_doubly_2014}.
To disprove this, we change the temperature of the target policy to $\beta_1 = -10$ and show new results in \cref{fig: UCI estimator selection beta 1 -10}.
The DR is no longer the best estimator, yet \ocvdr performs well.
Both of our methods outperform \slope and \pasif again.
We also observe that \slope performs worse in this experiment.
Since both IPS and DR are unbiased, their confidence intervals often overlap.
Therefore, \slope mostly chooses DR regardless of its performance.

\subsection{Hyper-Parameter Tuning}
\label{sec: experiment hyper-parameter tuning}

We also evaluate \ocv on the hyper-parameter tuning of seven estimators from \cref{sec: related work}.
We present them next.
TruncatedIPS \citep{ionides_truncated_2008} is parameterized by a clipping constant $M$ that clips higher propensity weights than $M$.
The authors suggest $M = \sqrt{n}$.
\switchdr \citep{wang_optimal_2017} has a threshold parameter $\tau$ that switches to DM if the propensity weights are too high and uses DR otherwise.
The authors propose their own tuning strategy by pessimistically bounding the estimator's bias \citep{wang_optimal_2017}.
CAB \citep{su_cab_2019} has a parameter $M$ that adaptively blends DM and IPS.
The authors do not propose any tuning method.
DRos and DRps \citep{su_doubly_2020} have a parameter $\lambda$ that regularizes propensity weights to decrease DR's variance.
The authors propose a tuning strategy similar to that of \switchdr.
IPS-$\lambda$ \citep{metelli_subgaussian_2021} has a parameter $\lambda$ that regularizes propensity weights while keeping the estimates differentiable, which is useful for off-policy learning.
The authors propose a differentiable tuning objective to get optimal $\lambda$.
GroupIPS \citep{peng_offline_2023} has multiple tuning parameters, such as the number of clusters $M$, the reward model class to identify similar actions, and the clustering algorithm.
The authors propose choosing $M$ by \slope.
We describe the estimators, their original tuning procedures, and hyper-parameter grids in \cref{sec: appendix tuned estimators details}.

All methods are evaluated in $90$ different conditions: $9$ UCI ML Repository datasets \citep{bache_uci_2013}, two target policies $\beta_1 \in \{-10, 10\}$, and five logging policies $\beta_0 \in \{-3, -1, 0, 1, 3\}$.
This covers a wide range of scenarios: logging and target policies can be close or differ, their values can be high or low, and dataset sizes vary from small ($107$) to larger ($10\,000$).
Each condition is repeated $5$ times, and we report the MSE over all runs and conditions in \cref{fig: tuning}.
We observe that theory-suggested hyper-parameter values generally perform the best if they exist.
Surprisingly, \ocv often matches their performance while also being a general solution that applies to any estimator.
It typically outperforms \slope and \pasif.

We also consider the problem of joint estimator selection and hyper-parameter tuning.
We evaluate \ocvdr, \ocvips, and \pasif on this task and report our results as \emph{Everything} in \cref{fig: tuning}.
\slope cannot be evaluated because the correct order of the estimators is unclear.
We observe that both of our estimators perform well and have an order of magnitude lower MSE than \pasif.
This shows that \ocv is a reliable and practical method.

%% file: 07_conclusion.tex
\section{Conclusion}

We propose an estimator selection and hyper-parameter tuning procedure for off-policy evaluation that uses cross-validation, bridging an important gap between off-policy evaluation and supervised learning.
Estimator selection in off-policy evaluation has been mostly theory-driven.
In contrast, in supervised learning, cross-validation is preferred despite limited theoretical support.
We overcome the issue of an unknown policy value by using an unbiased estimator on a held-out validation set.
This is similar to cross-validation in supervised learning, where we only have samples from an unknown distribution.
We test our method extensively on nine real-world datasets, as well as both estimator selection and hyper-parameter tuning tasks.
The method is widely applicable, simple to implement, and easy to understand since it relies on principal techniques from supervised learning that are well-known and loved by practitioners.
It also outperforms state-of-the-art methods.

One natural future direction is off-policy learning.
The main challenge is that the tuned hyper-parameters have to work well for any policy instead of a single target policy.
At the same time, naive tuning of some worst-case empirical risk could lead to too conservative choices.
Another potential direction is an extension to reinforcement learning.

%% file: 08_appendix.tex
\section{Implementation Details}
\label{sec: appendix implementation details}

\paragraph{Datasets}
All datasets are publicly available online. 
In particular, we use the OpenML dataset repository \citep{vanschoren_openml_2014}.
If there are multiple datasets with the same name, we always use version \emph{v.1}.
\cref{table: datasets} summarizes dataset statistics.\footnote{Our source code is available at \url{https://github.com/navarog/cross-validated-ope}}

\begin{table*}[t]
\caption{Characteristics of the datasets used in the experiments.}
\label{table: datasets}
\begin{center}
\begin{small}
\begin{tabular}{l|ccccccccc}
\toprule
Dataset & ecoli & glass & letter & optdigits & page-blocks & pendigits & satimage & vehicle & yeast\\
\midrule
Classes & 8 & 6 & 26 & 10 & 5 & 10 & 6 & 4 & 10\\
Features & 7 & 9 & 16 & 64 & 10 & 16 & 36 & 18 & 8\\
Sample size & 336 & 214 & 20000 & 5620 & 5473 & 10992 & 6435 & 846 & 1484\\
\bottomrule
\end{tabular}
\end{small}
\end{center}
\end{table*}

\begin{table*}
\caption{Hyper-parameters for the respective estimators resulting in the increasing variance order.}
\label{table: hyper-parameters variance order}
\begin{center}
\begin{small}
\begin{sc}
\resizebox{\linewidth}{!}{
\begin{tabular}{l|ccccccc}
\toprule
Estimator & TruncatedIPS & \switchdr & CAB & DRos & DRps & IPS-$\lambda$ & GroupIPS\\
Variance order & $w_{0.05} \leq w_{0.95}$ & $w_{0.05} \leq w_{0.95}$ & $w_{0.05} \leq w_{0.95}$ & $(w_{0.05})^2 \leq (w_{0.95})^2$ & $w_{0.05} \leq w_{0.95}$ & $h_{-10} \geq h_{10}$ & $M_{2} \leq M_{32}$\\
\bottomrule
\end{tabular}
}
\end{sc}
\end{small}
\end{center}
\end{table*}

\subsection{Estimators Tuned in the Experiments}
\label{sec: appendix tuned estimators details}
To simplify the notation in this section, we define propensity weights $\displaystyle w(x_i, a_i) = \frac{\pi(a_i \mid x_i)}{\pi_0(a_i \mid x_i)}$.

\paragraph{TruncatedIPS}
The \emph{truncated inverse propensity scores} estimator \citep{ionides_truncated_2008} introduces a clipping constant $M > 0$ to the IPS weights 
\begin{align}
\label{eq: truncated ips}
    \hat{V}_\text{TruncatedIPS}(\pi; \cD, M) = \frac{1}{n}\sum_{i=1}^n\min\set{M, w(x_i, a_i)}r_i.
\end{align}
This allows trading off bias and variance.
When $M = \infty$, TIPS reduces to IPS.
When $M = 0$, the estimator returns 0 for any policy $\pi$.
The theory suggests to set $M = \cO(\sqrt{n})$ \citep{ionides_truncated_2008}.
In our experiments, we search for $M$ on the hyperparameter grid of 30 geometrically spaced values. 
The smallest and largest $\tau$ values in the grid are $w_{0.05}$ and $w_{0.95}$, denoting the 0.05 and 0.95 quantiles of the propensity weights. 
We also include the theory-suggested value in the grid.

\paragraph{\switchdr}
The \emph{switch doubly-robust} estimator \citep{wang_optimal_2017} introduces a threshold parameter $\tau$ to ignore residuals of $\hat{f}(x_i, a_i)$ that have too large propensity weights
\begin{multline}
\label{eq: switch dr}
  \hat{V}_{\switchdr}(\pi; \cD, \tau)\\
  = \frac{1}{n}\sum_{i=1}^n \I{w(x_i, a_i) \leq \tau }w(x_i, a_i)
  (r_i - \hat{f}(x_i, a_i))\\
  + \hat{V}_\textsc{DM}(\pi; \cD).
\end{multline}
When $\tau = 0$, \switchdr becomes DM \eqref{eq: dm} whereas $\tau = \infty$ makes it DR \eqref{eq: dr}. 
The authors propose a tuning procedure where they conservatively upper bound bias of DM to the largest possible value for every data point. 
This is to preserve the minimax optimality of \switchdr with using estimated $\hat{\tau}$ as the threshold would only be activated if the propensity weights suffered even larger variance
\begin{align}
    \hat{\tau} = \argmin_\tau \var{\hat{V}_{\switchdr}(\pi; \cD, \tau)} + \text{Bias}_\tau^2\\
    \text{Bias}_\tau^2 = \left[\frac{1}{n}\sum_{i=1}^n\E[\pi]{R_\text{max} \I{w(x_i, a_i) > \tau} \mid x_i}\right]^2,
\end{align}
where the authors assume we know maximal reward value $0 \leq r(x, a) \leq R_\text{max}$, which in our experiments is set at $R_\text{max} = 1$.
We define the grid in our experiments similarly to that of TruncatedIPS, where the grid has 30 geometrically spaced values. 
The smallest and largest $\tau$ values in the grid are $w_{0.05}$ and $w_{0.95}$.
The authors originally proposed a grid of 21 values where the smallest and the largest values are the minimum and maximum of the propensity weights. 
We opted for the larger grid as we did not observe negative changes in the estimator's performance, and we want to keep the grid consistent with the subsequent estimators.

\paragraph{CAB}
The \emph{continuous adaptive blending} estimator \citep{su_cab_2019} weights IPS and DM parts based on propensity weights, where DM is preferred when the propensity weights are large and vice versa
\begin{multline}
\label{eq: cab}
    \hat{V}_\textsc{CAB}(\pi; \cD, M) = \frac{1}{n}\sum_{i=1}^n\sum_{a \in \cA}\pi(a \mid x_i)\alpha_i(a)\hat{f}(x_i, a)\\ 
    + \frac{1}{n}\sum_{i=1}^n w(x_i, a_i) \beta_i r_i,\\
    \alpha_i(a) = 1 - \min\set{Mw(x_i, a)^{-1}, 1},\\
    \beta_i = \min\set{M w(x_i, a_i)^{-1}, 1}.
\end{multline}
The estimator reduces to DM when $M = 0$ and to IPS when $M = \infty$.
The advantage is that this estimator is sub-differentiable, which allows it to be used for policy learning.
The authors do not propose any tuning procedure.
In our experiments, we search for $M$ on the hyperparameter grid of 30 geometrically spaced values with the smallest and largest $M$ values in the grid $w_{0.05}$ and $w_{0.95}$. 

\paragraph{DRos, DRps}
The \emph{doubly-robust estimators with optimistic and pessimistic shrinkages} \citep{su_doubly_2020} are the estimators that shrink the propensity weights to minimize a bound on the mean squared error
\begin{multline}
\label{eq: drs}
  \hat{V}_\textsc{DRs}(\pi; \cD, \lambda)
  = \frac{1}{n}\sum_{i=1}^n \hat{w}_\lambda(x_i, a_i)
  (r_i - \hat{f}(x_i, a_i))\\ + \hat{V}_\textsc{DM}(\pi; \cD),\\
  \hat{w}_{o,\lambda}(x, a) = \frac{\lambda}{w(x, a)^2 + \lambda} w(x, a)\\
  \hat{w}_{p, \lambda}(x, a) = \min\set{\lambda, w(x, a)},
\end{multline}
where $\hat{w}_{o,\lambda}$ and $\hat{w}_{p,\lambda}$ are the respective optimistic and pessimistic weight shrinking variants, and we refer to the estimators that use them as DRos and DRps.
In both cases, the estimator reduces to DM when $\lambda = 0$ and to DR when $\lambda = \infty$.
The authors also propose a tuning procedure to estimate $\hat{\lambda} = \argmin_\lambda \var{\hat{V}_\textsc{DRs}(\pi; \cD, \lambda)} + \text{Bias}_\lambda^2$ where they bound the bias as follows
\begin{align}
    \text{Bias}_\lambda^2 = \left[\frac{1}{n}\sum_{i=1}^n(\hat{w}_\lambda(x_i, a_i) - w(x_i, a_i))(r_i - \hat{f}(x_i, a_i)\right]^2.
\end{align}
Following the authors, our experiments define the hyper-parameter grid of 30 geometrically spaced values.
For DRos, the smallest and largest $\lambda$ values on the grid are $0.01 \times (w_{0.05})^2$ and $100 \times (w_{0.95})^2$ and for DRps, they are $w_{0.05}$ and $w_{0.95}$.

\paragraph{IPS-$\boldsymbol\lambda$}
The \emph{subgaussian inverse propensity scores} estimator \citep{metelli_subgaussian_2021} improves \emph{polynomial} concentration of IPS \eqref{eq: ips} to subgaussian by correcting the propensity weights
\begin{align}
\label{eq: sgips}
    \hat{V}_{\textsc{IPS-}\lambda}(\pi; \cD, \lambda) = \frac{1}{n}\sum_{i=1}^n \frac{w(x_i, a_i)}{1 - \lambda + \lambda w(x_i, a_i)}r_i.
\end{align}
When $\lambda = 0$, the estimator reduces to IPS, and when $\lambda = 1$, the estimator returns 1 for any $\pi$.
Note that this is the \emph{harmonic} correction, while a more general definition uses propensity weights $\displaystyle w_{\lambda, s}(x, a) = \left((1 - \lambda)w(x, a)^s + \lambda\right)^{\frac{1}{s}}$.
The authors also propose a tuning procedure where they choose $\lambda$ by solving the following equation
\begin{align}
    \lambda^2 \frac{1}{n}\sum_{i=1}^n w_{\lambda, \sqrt[4]{n}}(x_i, a_i)^2 = \frac{2\log\frac{1}{\delta}}{3n}.
\end{align}
The equation uses a general definition of $w_{\lambda, s}$ where $s = \sqrt[4]{n}$ and can be solved by gradient descent.
As this parameter has an analytic solution, the authors did not specify any hyper-parameter grid.
We define the grid to be $(1+\exp(-x))^{-1}_{(h \in [-10, 10])}$ where $(h \in [-10, 10])$ are 30 linearly spaced values.

\paragraph{GroupIPS}
The \emph{outcome-oriented action grouping IPS} estimator \citep{peng_offline_2023} has multiple parameters: the reward model class, the clustering algorithm, and the number of clusters. 
In GroupIPS, one first learns a reward model $\hat{f}(x, a)$ to estimate the mean reward.
The reward model is then used to identify similar actions.
The authors \citep{peng_offline_2023}, in their experiments, use a neural network for it.
We simplify it in line with other baselines.
We learn the same ridge regression model and use the estimated mean reward $\hat{f}(x, a)$ to group context-action pairs.
More formally, a clustering algorithm $\cG$ learns a mapping $g = \cG(\cD, \hat{f}, M)$ that assigns each context-action pair $(x, a) \in \cD$ to a cluster $m \in M$ based on its estimated mean reward $\hat{f}(x, a)$. 
While the authors originally used K-means clustering, in our experiments, we use uniform binning as it is computationally more efficient. 
We split the reward space $[0, 1]$ into $M$ equally spaced intervals and assign each context-action pair to the corresponding interval based on its estimated mean reward.
Finally, IPS is used to reweight the policies based on the propensity weights of each cluster
\begin{align}
    \label{eq: group ips}
    \hat{V}_\text{GroupIPS}(\pi; \cD, M) = \frac{1}{n}\sum_{i=1}^n \frac{\pi(g(x_i, a_i) \mid x_i)}{\pi_0(g(x_i, a_i) \mid x_i)}r_i
\end{align}
where $\pi(m, x) = \sum_{a \in \cA}\I{g((x, a) = m}\pi(a \mid x)$ is a shorthand for the conditional probability of recommending any action mapped to cluster $m$ in context $x$.
We still need to tune $M$, and the prior work \citep{peng_offline_2023} uses \slope for it. 
In our experiments, we define the hyper-parameter grid for the number of clusters $M$ as $\set{2, 4, 8, 16, 32}$.

\subsection{Estimator Selection and Hyper-Parameter Tuning}
\label{sec: appendix tuning methods details}

\paragraph{Variance estimation}
\label{sec: appendix variance estimation}
\slope and \ocv estimate the estimator's variance as part of their algorithm.
\slope derives the estimator's confidence intervals from it, and \ocv uses it to set the optimal training/validation ratio.

All estimators in \cref{sec: off-policy evaluation,sec: appendix tuned estimators details} are averaging over $n$ observations $(x_i, a_i, r_i)$ and we use this fact for variance estimation in line with other works \citep{wang_optimal_2017, su_adaptive_2020}.
We illustrate this on TruncatedIPS.
Let $\hat{v}_i(M) = \min\set{M, w(x_i, a_i)}r_i$ be the value estimate of $\pi$ for a single observation $(x_i, a_i, r_i)$ and averaging over it leads to $\bar{v}_M = \hat{V}_\text{TruncatedIPS}(\pi; \cD, M) = \frac{1}{n}\sum_{i=1}^n \hat{v}_i(M)$ . 
Since $x_i$ are i.i.d., the variance can be estimated as 
\begin{align}
    \sigma^2_M \approx \frac{1}{n^2}\sum_{i=1}^n\left(\hat{v}_i(M) - \bar{v}_M\right)^2.
\end{align}
Using this technique in \slope, we get the 95\% confidence intervals as $\left[\bar{v}_M - 2\sigma_M, \bar{v}_M + 2\sigma_M\right]$.
This is also valid estimate in our experiments since $r_i \in [0, 1]$, all policies are constrained to the class defined in \eqref{eq: experiments policy softmax}, which ensures $\pi_0$ has full support, hence $\hat{v}_i$ are bounded. 

\paragraph{\slope}
We use 95\% confidence intervals according to the original work of \citet{su_adaptive_2020}.
In \cref{sec: experiment estimator selection}, we use the order of the estimators $\var{\hat{V}_\textsc{IPS}(\pi)} \geq \var{\hat{V}_\textsc{DR}(\pi)} \geq \var{\hat{V}_\textsc{DM}(\pi)}$.
The order of the hyper-parameter values for seven estimators tuned in \cref{sec: experiment hyper-parameter tuning} is summarized in \cref{table: hyper-parameters variance order}.

Except for IPS-$\lambda$, a hyper-parameter of a higher value results in a higher-variance estimator; hence, the algorithm starts with these.

\paragraph{PAS-IF}
The tuning procedure of \pasif uses a neural network to split the dataset and create surrogate policies.
We modify the original code from the authors' GitHub to speed up the execution and improve stability.
We use the same architecture as the authors \citep{udagawa_policy-adaptive_2023}, a 3-layer MLP with 100 neurons in each hidden layer and ReLU activation.
We observed numerical instabilities on some of our datasets. 
Hence, we added a 20\% dropout and batch normalization after each hidden layer.
The final layer has sigmoid activation.
We use Adam as its optimizer.
We use a batch size of 1000, whereas the authors used 2000.
The loss function consists of two terms, $\cL = \cL_d + \alpha \cL_r$, where $\cL_d$ forces the model to output the propensity weights of surrogate policies that match the original $(w(x_i, a_i))_{i \in [n]}$, and $\cL_r$ forces the model to split the dataset using 80/20 training/validation ratio.
The authors iteratively increase the coefficient $\alpha \in [0.1, 1, 10, 100, 1000]$ until the resulting training/validation ratio is 80/20 $\pm$ 2.
This results in a lot of computation overhead; hence, we dynamically set $\alpha = 0$ if the ratio is within 80/20 $\pm$ 2 and $\alpha = 1000$ otherwise.
Finally, if the logging and target policies substantially differ, the propensity weights $w(x, a)$ are too large, and the original loss function becomes numerically unstable.
Hence, we clip the target weights $\min\set{w(x, a), 10^{7}}$.
The clipping is loose enough not to alter the algorithm's performance but enough to keep the loss numerically stable.
Finally, we run \pasif for 5000 epochs as proposed by the authors, but we added early stopping at five epochs (the tolerance set at $10^{-3}$), and the algorithm usually converges within 100 epochs.

\section{Additional Experiments}
\label{sec: appendix additional experiments}
In these experiments, we perform ablation on the individual components of our method to empirically support our decisions, namely theory-driven training/validation split ratio and the one standard error rule.
We also ablate the number of repeated splits $K$ and show that a higher number improves the downstream performance, but we observe diminishing returns as the results observed in repeated splits are correlated.
We also discuss the importance of choosing an unbiased validator, and we empirically show DM as a validator performs poorly. 
Additionally, we discuss computational complexity of our methods.

\paragraph{Our improvements make standard cross-validation more stable}
Our method has two additional components to reduce the variance of validation error: the training/validation split ratio and one standard error rule. 
We discuss them in \cref{sec: analysis,sec: one standard error rule}.
We ablate our method, gradually adding these components.
We use the same setup as in \cref{fig: UCI estimator selection}, average the results over all datasets, and report them in \cref{fig: UCI gradual improvements}.
We start with the standard 10-fold cross-validation where different validation splits do not overlap; hence, the training/validation ratio is set at 90/10.
We also choose the estimator with the lowest mean squared error, not the lowest upper bound.
In \cref{fig: UCI gradual improvements}, we call this method A: \ocvdr 90/10 split ratio.
Next, we change the selection criterion from the mean loss to the upper bound on mean loss (B: A + one SE rule).
We observe dramatic improvements, making the method more robust so it does not choose the worst estimator.
Then, instead, we try our adaptive split ratio as suggested in \cref{sec: analysis} and see this yields even bigger improvements (C: A + theory split ratio).
We then combine these two improvements together (D: B + C).
This corresponds to the method we use in all other experiments.
We see the one standard error rule does not give any additional improvements anymore, as our theory-driven training/validation ratio probably results in similarly-sized confidence intervals on the estimator's MSE.  
Additionally, as the theory-suggested ratio is not dependent on $K$ number of splits, we also change it to $K=100$, showing this gives additional marginal improvements (E: D + 100 K training/test splits).

\begin{figure}
    \centering
    \includegraphics[width=\linewidth]{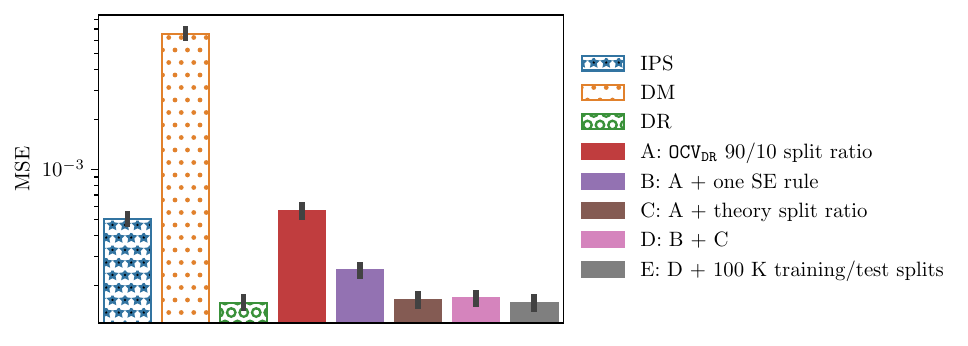}
    \caption{Ablation on proposed improvements from \cref{sec: analysis,sec: one standard error rule} with \ocvdr. This shows that both improvements individually help reduce the variance of estimation errors. However, when combined, the theory split ratio makes the one standard error rule insignificant.}
    \label{fig: UCI gradual improvements}
\end{figure}

\begin{figure}
    \centering
    \includegraphics[width=\linewidth]{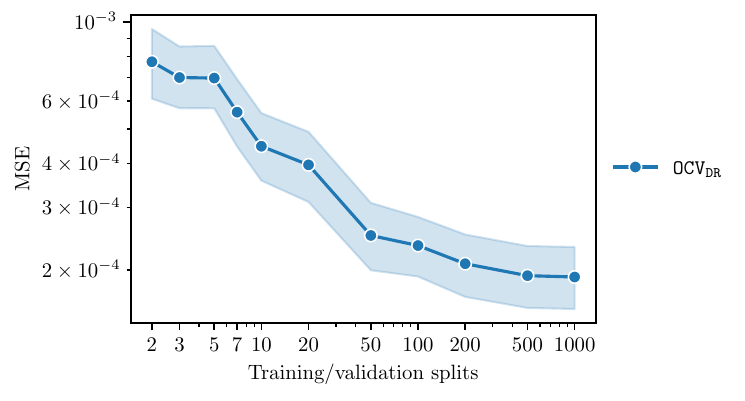}
    \caption{Ablation of the number of repeated training/validation splits with \ocvdr on the \emph{vehicle} dataset averaged over 500 runs. This shows us diminishing improvements as we increase the number of splits.}
    \label{fig: UCI increasing inner cycles}
\end{figure}

\begin{figure*}
    \centering
    \includegraphics[width=\linewidth]{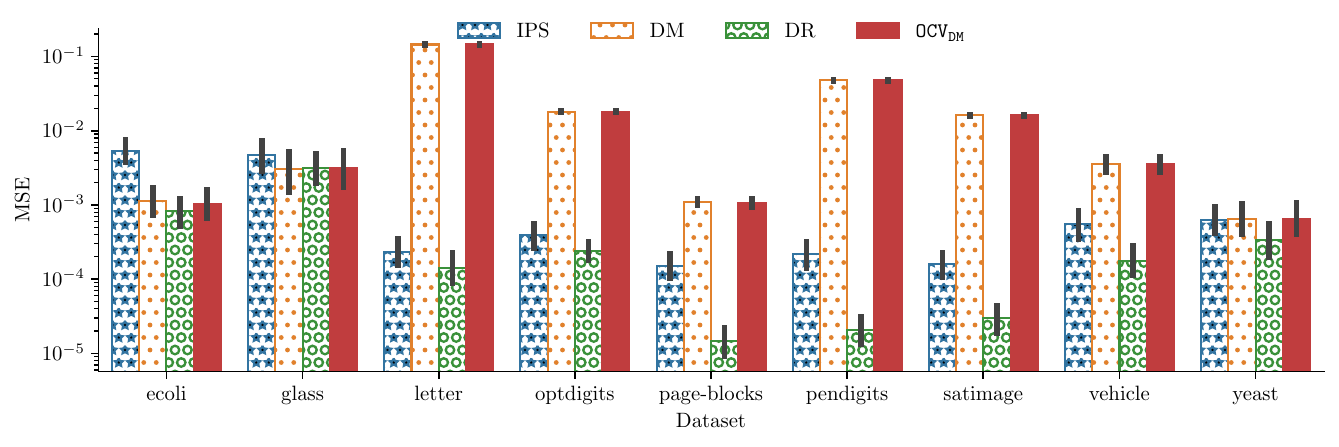}
    \caption{MSE of cross-validation, when using DM as a validator. As DM is \emph{biased}, it systematically chooses an estimator that is biased in the same direction: DM itself.}
    \label{fig: UCI estimator selection CV_DM}
\end{figure*}

We show in more detail in \cref{fig: UCI increasing inner cycles} how the CV performance improves with the increasing number of training/validation splits.
As expected, there are diminishing returns with an increasing number of splits. 
As the splits are correlated, there is an error limit towards which our method converges with increasing $K$. 

\paragraph{The validator used in cross-validation has to be unbiased}
In \cref{sec: method}, we design our method so that the validator has minimal bias. 
That is why we use classes of unbiased estimators, such as IPS \eqref{eq: ips} and DR \eqref{eq: dr}.
Otherwise our optimization objective would be shifted to prefer the estimators biased in the same direction.
This might be the case of poor \pasif's performance as its estimate on the validation set is not unbiased. 
We demonstrate this behavior on the same experimental setup as in \cref{fig: UCI estimator selection}.
We use DM as $\tilde{V}$ and report the results in \cref{fig: UCI estimator selection CV_DM} averaged over 100 independent runs.
The estimator selection procedure of \ocvdm is biased in the same direction as the DM estimator, and the procedure selects it even though it performs poorly.
To compare it with \ocvdr, we see DR performs poorly in \cref{fig: UCI estimator selection beta 1 -10}, especially on the \emph{glass} dataset. 
However, \ocvdr still performs.

\begin{figure}
    \centering
    \includegraphics[width=\linewidth]{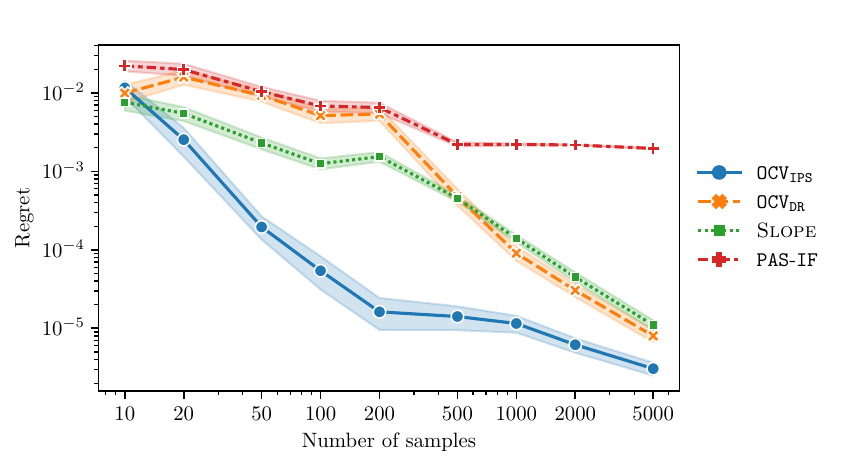}
    \caption{Regret of the estimator selection methods that choose between IPS, DM, and DR on a subsampled \textit{satimage} dataset. \ocv performs well even in low-data regimes.}
    \label{fig: regret satimage}
\end{figure}

\paragraph{\ocv can outperform \slope even in low-data regimes}
We ablate the number of samples in $\cD$ and observe how well the methods choose between IPS, DM, and DR for a given sample size.
We measure $\text{Regret} = L(\hat{V}_*) - L(V_*)$, a difference between the loss of the chosen estimator $\hat{V}_*$ and the optimal estimator $V_*$ that would get the minimal loss in a given run.
$L$ is squared error defined as in \eqref{eq: loss ope}.
We choose the \textit{satimage} dataset as it is the least computationally expensive dataset that is large enough to perform this ablation. 
The experiment is run as described in \cref{sec: experiments}, with $\beta_0, \beta_1 = (1, -1)$ and averaged over $500$ runs.
The results in \cref{fig: regret satimage} confirm our intuition that estimator selection gets more precise with more data.
\ocvips outperforms \slope even at low-data regimes because \slope relies on confidence intervals, which become wide.

\begin{table}
\begin{center}
\begin{small}
\begin{sc}
\begin{tabular}{l|cccc}
\toprule
Method & \ocvips & \ocvdr & \slope & \pasif\\
Time & 0.06s & 0.13s & 0.005s & 13.91s\\
\bottomrule
\end{tabular}
\end{sc}
\end{small}
\end{center}
\caption{Average computational cost of a single policy evaluation from \cref{fig: UCI estimator selection} when doing $K = 10$ training/validation splits with \ocvdr, \ocvips, and \pasif. Computed on AMD Ryzen 9 8945HS and NVIDIA GeForce RTX 4070 Laptop.}
\label{table: computational cost}
\end{table}

\paragraph{Cross-validation is computationally efficient}
To split the dataset, \pasif has to solve a complex optimization problem using a neural network.
This is computationally costly and sensitive to tuning.
We tuned the neural network architecture and loss function to improve the convergence and stability of \pasif.
We also run it on a dedicated GPU and provide more details in \cref{sec: appendix tuning methods details}.
Despite this, our methods are $100$ times less computationally costly than \pasif (\cref{table: computational cost}).